\documentclass[a4paper,conference]{IEEEtran}
\usepackage[utf8]{inputenc} % allow utf-8 input
\usepackage[T1]{fontenc}    % use 8-bit T1 fonts
\usepackage{hyperref}       % hyperlinks
\usepackage{url}            % simple URL typesetting
\usepackage{booktabs}       % professional-quality tables
\usepackage{amsfonts}       % blackboard math symbols
\usepackage{nicefrac}       % compact symbols for 1/2, etc.
\usepackage{microtype}      % microtypography
\usepackage{algorithm}
\usepackage[ruled,vlined,algo2e]{algorithm2e}
\usepackage{setspace}
\usepackage{fancyhdr,graphicx,amsmath,amssymb}
\usepackage{amsthm}
\usepackage{multirow}
\usepackage{xcolor}
\newtheorem{prop}{Proposition}
\usepackage{cite}

\begin{document}
    
\title{TreeRNN: Topology-Preserving Deep Graph Embedding and Learning}

\author{\IEEEauthorblockN{Yecheng Lyu, Ming Li, Xinming Huang, Ulkuhan Guler, Patrick Schaumont, and Ziming Zhang}
\IEEEauthorblockA{Electrical and Computer Engineering\\
Worcester Polytechnic Institute\\
\{ylyu, mli12, xhuang, uguler, pschaumont, zzhang15\}@wpi.edu}
}

\maketitle

\begin{abstract}
General graphs are difficult for learning due to their irregular structures. Existing works employ message passing along graph edges to extract local patterns using customized graph kernels, but few of them are effective for the integration of such local patterns into global features. In contrast, in this paper we study the methods to transfer the graphs into trees so that explicit orders are learned to direct the feature integration from local to global. To this end, we apply the breadth first search (BFS) to construct trees from the graphs, which adds direction to the graph edges from the center node to the peripheral nodes. In addition, we proposed a novel projection scheme that transfer the trees to image representations, which is suitable for conventional convolution neural networks (CNNs) and recurrent neural networks (RNNs). To best learn the patterns from the graph-tree-images, we propose TreeRNN, a 2D RNN architecture that recurrently integrates the image pixels by rows and columns to help classify the graph categories. We evaluate the proposed method on several graph classification datasets, and manage to demonstrate comparable accuracy with the state-of-the-art on MUTAG, PTC-MR and NCI1 datasets.
\end{abstract}

\IEEEpeerreviewmaketitle

\section{Introduction}\label{Introduction}
Deep graph learning has been attracting increasing research interests in recent years \cite{Fey2019PytorchGeometric}\cite{kipf2016GCN}\cite{simonovsky2017ECConv}\cite{zhang2018DGCNN}\cite{hamilton2017GraphSAGE}\cite{xu2018GINConv}\cite{morris2019GraphConv}\cite{gilmer2017message_passing}\cite{jiang2019gic}\cite{cao2016DNGR}\cite{al2019ddgk}\cite{niepert2016PSCN}. As a widely used data structure to store the topological features, a graph saves the point features in a node list and their affiliations as node edges. The nodes in a graph are orderless and the affiliations are sparse, which makes it difficult for deep graph learning. Trees are ordered graphs with a clear hierarchy, but they still fail to serve as tensors for network processing. In contrast, images have a tensor-like structure with densely ordered pixels in local regions. Such local regularity is beneficial for fast convolutions and recurrent processing that efficiently and effectively learn the local pattern from pixels within different applications.

\textbf{Motivation.}
Existing graph neural networks (GNNs) try to collect the features from adjacent nodes through message passing \cite{gilmer2017message_passing} to perceive local pattern. In GNN kernels, the center node plays the same role as adjacent nodes or is just a bit higher weighted in convolution. Surprisingly, we find that there are few existing works contributing to the hierarchy that integrates the local features to global features. In addition, existing GNN works mainly focus on specialized graph convolution kernels, which cannot benefit from the conventional neural networks. These observations motivate us to address the following question:
\begin{center} 
  \em{How to effectively and efficiently project graphs into an ordered and regularized space so that we can take advantage of pattern extraction in conventional neural networks for graph learning?}
\end{center}

\begin{figure*}[h]
\centering
\includegraphics[width=\textwidth]{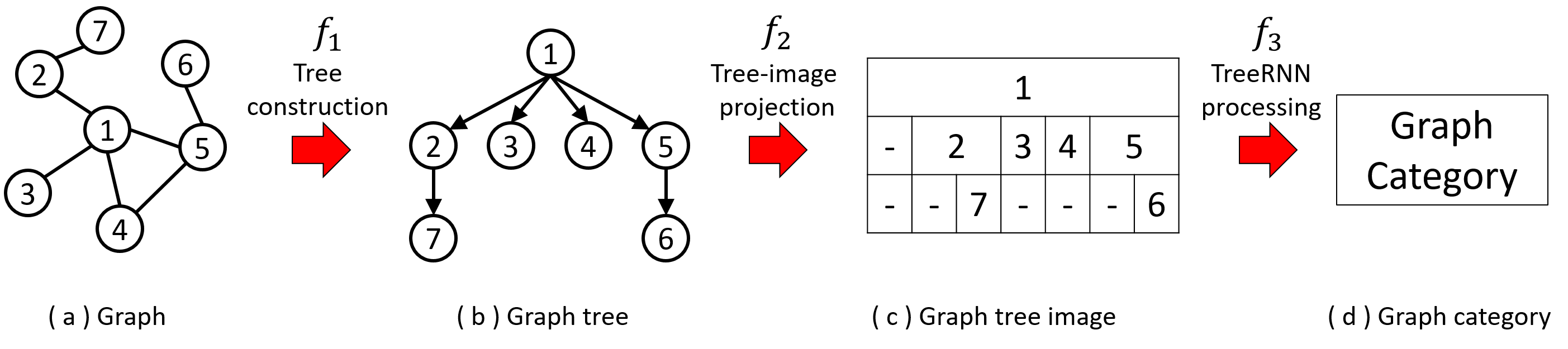}
 % \vspace{-1em}
\caption{System overview of the proposed graph classification method. (a) a given graph; (b) a tree construct from the corresponding graph; (c) an image representation projected from the graph-tree; (d) graph category predicted by the TreeRNN from the graph-tree-image. \label{fig:overview}}
\end{figure*}
 % \vspace{-1em}

\textbf{Approach.}
The question above is critical. A bad projection function easily leads to the loss of topological information in a graph with, for instance, misplaced parents and children in a tree. Such topological loss is fatal as it may introduce so much noise that the pattern is completely changed. Therefore, a good graph projection function is the key to ordered and regularized representations of input graphs. 

At the system level, the pipeline of our method can be summarized as follows: (1) construct trees from graphs, (2) project the trees into image space, and (3) classify graph-tree-images using TreeRNN, our proposed novel network architecture.

We are motivated by the DeepWalk \cite{perozzi2014deepwalk} that generates a batch of ordered node list from a graph. In DeepWalk, a random walk scheme is applied to a graph, which starts from a selected root node and goes through the graph edges for several steps, resulting in a list of nodes it passed. By applying the random walker multiple times starting from the same node, a batch of node lists is constructed and reformed as an image for network processing. However, the random walk fails to guarantee the tree to coverage over all graph nodes, and it fails to stop the walker visiting a node more than once. Those failures in topological-preserving confuse the neural network to encode the graph topology. In contrast, we propose to construct trees, a kind of directed acyclic graph, which guarantee coverage to all graph nodes and no repeated nodes on the tree. Within this projection, nodes are clearly ordered in the tree space, which helps extract the local topological features. To further transfer the tree to a structure that is feasible to conventional CNN and RNN, motivated by the Ordered Neurons \cite{shen2019ordered}, we employ a block-view like projection from the tree to image space. The block view image explicitly presents the hierarchy of a tree in the pixel context, which contributes to better feature extraction and classification.

Now the graph-tree-image are ready for conventional CNN and RNN processing. To better take advantage of the order and regularity in graph-tree-images, we propose TreeRNN, a novel RNN architecture that integrates the pixels in the images following the tree structures. Specifically, we employ a vanilla RNN unit and designed a novel network module to achieve 2D recurrent integration on image rows and columns by turns. The pixels in the same row represent the graph nodes on the same layer in the tree, while the pixels in the same column represent the graph nodes connected across the tree layers. By employing this novel network module, we succeed to extract features form graph-tree-images within few parameters, which makes the TreeRNN light-weighted.

\textbf{Contribution.}
In summary, our key contributions in this paper are as follows:
\begin{itemize}
	\item We are the {\em first}, to the best of our knowledge, to explore the graph-tree-image projection in the context of graph classification.
	\item We accordingly propose TreeRNN, a novel RNN architecture to process the graph-tree-images that recurrently process on image rows and columns in turns, which implicitly passes through the tree structure.
	\item We apply the integrated method to the graph classification application and experiment on three graph classification datasets named MUTAG, PTC and NCI1. Our work results in comparable performance with the state-of-the-art works on those benchmark datasets, which demonstrate the success of our graph-tree-image projection scheme and TreeRNN architecture.
\end{itemize}

\section{Related Work}\label{Related_Work}
\textbf{Graph Embedding.}
Graph Embedding has been studied for decades \cite{hamilton2017representation}, whose goal is to find a low dimensional representation of the graph nodes in some metric space so that the given similarity (or distance) function is preserved as much as possible. Force-directed graph layout is a series of graph embedding algorithms that try to project a graph to a 2D plane while preserving the distance between graph nodes using a force-directed function. Widely used methods in that series are Kamada-Kawai \cite{kamada198Kamada_Kawai}, Fruchterman-Reingold \cite{fruchterman1991graph} and FM$^3$ \cite{meyerhenke2015drawing}. These algorithms achieved great success in graph visualization. For graph embedding aimed at deep learning, however, those methods result in a high topological disparity on complicated graph embedding. In a recent survey paper \cite{cai2018comprehensive}, a comprehensive understanding of graph embedding techniques is introduced including the problems, techniques, and applications.

In our paper, different from the graph embedding methods that try to represent the graphs in low-dimension space for visualization, we propose an ordered and regular representation of a graph so that conventional convolution kernels and recurrent operators can apply to it.

\textbf{Tree Construction from Graphs.}
Tree construction from graphs tries to generate a connected and directed sub-graph with no cycles. Minimum spanning tree (MST) \cite{eisner1997state} is a kind of tree construction method that generates a tree with a minimum sum of edge weights. Graph tree search is another kind of method that traverses all graph nodes from a selected root node and generates a tree along the path. Depth-first search (DFS) \cite{tarjan1972depth} explores as far as possible along each branch before backtracking, while breadth-first search (BFS) \cite{bundy1984breadth} explores all of the neighbor nodes at the present depth before moving on to the nodes at the next depth level. Other tree construction methods includes K-MST \cite{zelikovsky1993minimal}, AVL tree \cite{adel1962AVL} and B-tree \cite{bayer2002BTree}.

In this paper, we employ the BFS to construct trees from graphs because it covers all the connected graph nodes within the fewest layers, which minimizes the memory allocation of the image representations described in Section \ref{subsec:f2}. 

\begin{figure*}[h]
\centering
\includegraphics[width=0.9\textwidth]{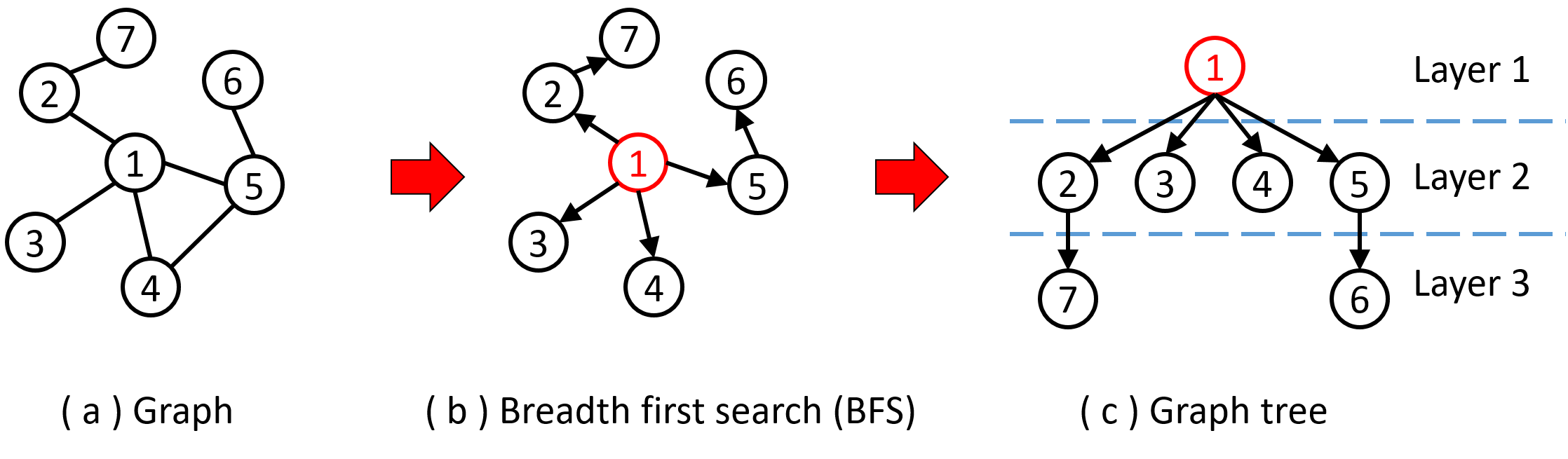}
\caption{Illustration of function $f_1:$  Tree construction from a given graph using breath first search. (a) a given graph; (b) apply breadth first search to the graph from its center node; (c) a constructed tree from the graph. \label{fig:f1}}
\end{figure*}

\textbf{Deep Graph Learning.}
Deep graph learning is an extension of conventional deep learning algorithms to the graph, an orderless and irregular data structure connecting the paired nodes with edges. Graph convolution \cite{Fey2019PytorchGeometric} is a big family of deep graph learning that fuses the local graph subsets by collecting the adjacent nodes' features via message passing and applying a pooling operation (max, sum or average) on them. GCN \cite{kipf2016GCN}, DGCNN \cite{zhang2018DGCNN}, ECConv \cite{simonovsky2017ECConv}, GraphSAGE \cite{hamilton2017GraphSAGE}, GraphConv \cite{morris2019GraphConv}, and GINConv \cite{xu2018GINConv} are good examples in this family. Another family of deep graph learning methods embed graphs into other feature spaces followed conventional neural network processing. This family includes DeepWalk \cite{perozzi2014deepwalk}, DDGK \cite{al2019ddgk}, DNGR \cite{cao2016DNGR}, PSCN \cite{niepert2016PSCN}, LINE \cite{tang2015line}, M-NMF \cite{wang2017M-NMF}, and WKPI \cite{zhao2019WKPI}. Graph embedding also works on graph-related data structures \emph{i.e.} Ordered Neurons \cite{shen2019ordered} on trees and Lyu \emph{et al.} \cite{lyu2020learning} on point clouds. In our paper, we follow the graph embedding family and embed the graphs to image space via a graph-tree-image projection.

\section{System Overview}\label{Overview}
In this paper, We focus on the problem of graph classification. Generally, let $\mathcal{G}\{V,E,X,Z\}$ denotes a graph where $V$ denotes the set of graph nodes, $E \subseteq (V \times V)$ denotes the set of graph edges, $X \in \mathcal{R}^{|V|\times S_V}$ denotes the set of node features with feature size $S_V$, and $Z \in \mathcal{R}^{|E|\times S_E}$ denotes the set of edge features with feature size $S_E$.

In the traditional graph learning setting, we learn a projection function $f: \mathcal{G} \rightarrow \mathcal{Y} \in \mathcal{F}$ that map a graph to one of the semantic labels. In our case, we take three steps to achieve the goal. Firstly, we propose a projection function $f_1: \mathcal{G} \rightarrow \mathcal{T}$ where $\mathcal{T}$ denotes a tree. The function is aimed to transfer the graph to tree space where the nodes are ordered by directed edges. Secondly, we proposed projection function $f_2: \mathcal{T} \rightarrow \mathcal{I}$ that further transfer the graph from the tree space $\mathcal{T}$ to image space $\mathcal{I}$, in which each graph node is mapped to the one or multiple pixels. Thirdly and lastly, we learn a neural network classifier $f_3: \mathcal{I} \rightarrow \mathcal{Y} \in \mathcal{F}_3$ that predicts the graph labels by classifying the graph-tree-images. The neural network classifier learns to minimizing certain loss function $\ell = \ell(f_3(\mathcal{I}),\mathcal{Y})$.
The pipeline is illustrated in Figure \ref{fig:overview}. Please note that $f_1$ and $f_2$ are non-trainable functions, and $f_3$ is learned from its parameter space $\mathcal{F}_3$.

\section{Projection from graph to tree and image space}
\subsection{Tree Construction from Graphs} \label{subsec:f1}
Tree construction from graphs has been well studied in graph theory. In our work, a tree constructed from a graph denotes a rooted directed acyclic graph (DAG) that contains all graph nodes and a subset of graph edges. A tree representation of a graph have two advantages: (1) it is rooted and directed, which contributes to the context feature extraction by order-sensitive operators such as convolutional kernels and recurrent units; (2) it has no cycles so that every node is visited only once along the tree, which helps eliminate confusion to the graph structure during feature extraction. DeepWalk \cite{perozzi2014deepwalk} also generates rooted, and directed subgraphs, however, it allows cycles and even self-loops that results in repeated visits to a node, which makes the feature extractor confused to learn the global features, \emph{i.e. how many nodes are there in the graph?}
Figure \ref{fig:f1} illustrates the steps to construct a tree from the input graph.

In our work, we employ the breath first search (BFS) to accomplish the first projection function $f_1: \mathcal{G} \rightarrow \mathcal{T}$. Comparing to the other tree construction methods such as depth first search (DFS) and minimum spanning tree (MST), the BFS traverses the graph within the least depth from the select depth. To minimize the tree depth constructed from the graph, we set the root node to the one with shortest length to its farthest node. The tree construction scheme is described in Algorithm \ref{Alg:f1}. Dijkstra in the algorithm denotes the Dijkstra Algorithm \cite{dijkstra1959dijkstra} that calculates the distance between each node pairs in the graph.

\begin{algorithm}[h]\footnotesize
    \KwIn{Graph $G \in \mathcal{G}$}
    \KwOut{Tree representation $T \in \mathcal{T}$}
    \quad \\
    $A \leftarrow adjacency\_matrix(G)$; \\
    $H \leftarrow dijkstra(A) $;\\
    $root \leftarrow argmin_x(\max_y(H(x,y)))$;\\
    $T \leftarrow BFS(G,root)$;\\
    \textbf{Return} $T$
    
	\caption{{$f_1: \mathcal{G} \rightarrow \mathcal{T}$ Tree Construction from Graph} \label{Alg:f1}}
\end{algorithm}

\subsection{Projection from Trees to Images} \label{subsec:f2}
The projection function $f_1$ constructs a directed and acyclic tree from a general graph. However, the tree structure is still non-feasible for conventional neural network processing. Hence, another projection function $f2: \mathcal{T} \rightarrow \mathcal{I}$ is proposed to further transfer the tree to an image-like array, which is feasible for network processing.

Given a set of tree ${T}$ with maximum node size $|V|_{max}$ and maximum depth $D_{max}$, the projection function $f_2$ is aimed to map all the nodes in each tree to a fixed-sized image space while preserving its topology. Specifically, there are two topological features we expect to keep: (1) child nodes connected to the same parent node are expected to be connected after projection to image space, and (2) each node is also expected to keep their adjacency to its parent node in the target image space to avoid confusion during network processing, those two adjacency should distinguish to each other. Considering the memory efficiency, we also want to limit the image space to a suitable size.

In our work, inspired by the block view projection in the DeepWalk \cite{perozzi2014deepwalk}, we propose a similar projection $f_2$ from tree space to image space. The projection is illustrated in Figure \ref{fig:f2}. As we see, graph nodes in each layer of the tree occupy a row in the image and each node covers pixels as many as its descendant node size plus one representing itself. Child nodes connected to the same parent node are connected to each other in the same row, while each child node is next to its root node in the same column, which satisfies the expectations we proposed before.

\begin{figure}[h]
\centering
\includegraphics[width=0.95\columnwidth]{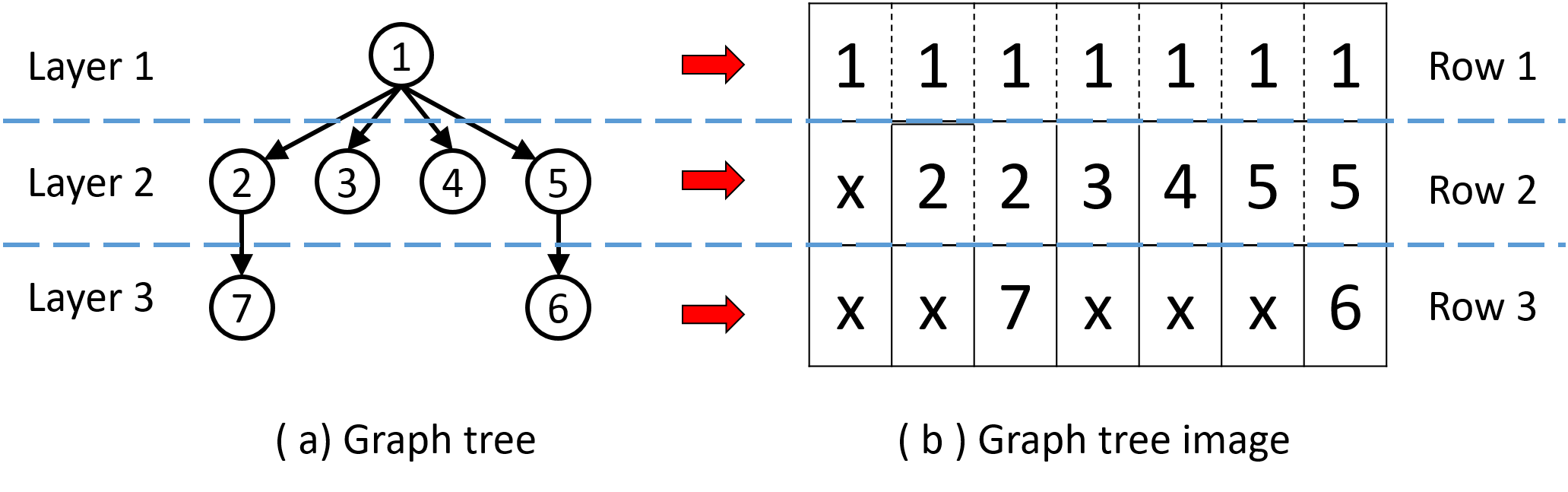}
\caption{Illustration of $f_2:$ projection of a graph from tree space to image space. (a) a given graph-tree from function $f_1$; (b) an image representation projected from the graph tree, Each pixel of the image contains the \textbf{node features} of the projected graph node and the \textbf{edge features} connecting the node and its parent. \label{fig:f2}}
\end{figure}

To determine the required image size to store the projected trees, we have to find the maximum rows and columns in need to project the graph set ${\mathcal{T}}$. According to $f_2$, the minimum number of rows $N_{row}(T)$ required for input $\mathcal{T}$ equals its tree depth $D(\mathcal{T})$. For the required number of columns $N_{col}(T)$, we introduce the following proposition to determine.

\begin{prop}\label{prop:1}
    Given a tree $\mathcal{T}$ defined in Section \ref{subsec:f1}, its required columns $N_{col}(T)$ is not less than to its node size $|V|$, meaning $N_{col}(T) \geq |V(T)|$.
\end{prop}
\begin{proof}
    Given a tree $\mathcal{T}$, its required columns for the first row (root layer) $n_1 = n(root)$ equals the sum of required columns for its child nodes plus itself, meaning
    \begin{align}\label{eqn:row1_raw}
    n_1 = n(root) = \sum_{v \in leaf(root)}{n(v)} + 1.
    \end{align}
    While the set of leaf nodes of root node is exactly the set of nodes in the second row, and root is the only root in the first layer, meaning
    \begin{align}\label{eqn:row1}
    n_1 = \sum_{v \in leaf(root)}{n(v)} + 1 = n_2 + |V|_{L_1}.
    \end{align}
    Let us extend the Eqn. \ref{eqn:row1} to other rows, we have
    \begin{align}\label{eqn:rowi}
    n_i = n_{i+1} + |V|_{L_i},
    \end{align}
    and for the last row,
    \begin{align}\label{eqn:row_last}
    n_{D(T)} = |V|_{L_{D(T)}}.
    \end{align}
    From Eqn. \ref{eqn:rowi} we conclude:
    \begin{align}\label{eqn:row_leq}
    n_i = n_{i+1} + |V|_{L_i} \geq n_{i+1}.
    \end{align} 
    Combine Eqn. \ref{eqn:row1} to \ref{eqn:row_leq}, we have
    \begin{align}\label{eqn:row_sum}
    N_{col}(T) \geq |V|_{L_1} + |V|_{L_2} + \dots +|V|_{L_{D(T)}} = |V(T)|.
    \end{align} 
    
\end{proof}

From Prop. \ref{prop:1} we conclude that the required image size $|\mathcal{I}|$ for graph set ${\mathcal{G}}$ equals $|V|_{max} \times D_{max}$. A detailed projection function is described in Algorithm \ref{Alg:f2}.

\begin{algorithm}[h]\footnotesize
    \KwIn{Tree $T \in \mathcal{T}$, Image space $|V|_{max} \times D_{max}$ }
    \KwOut{Image representation $I \in \mathcal{I}$}
    \quad \\
    $L_1 = [root(T)]$\\
    \ForAll {$i \in 1,2,\dots,D(T)$}{%
      $\;\;  i_{col} = 1$\\
      $L_{i+1} \leftarrow \varnothing $\\
      \ForAll {$v \in L_i$}{%
          \If{$v \neq \varnothing$}{
            $\;\; size_v \leftarrow ChildSize(v)$\\
            $I(i,i_{col}:i_{col}+size_v) \leftarrow v$\\
            $i_{col} \leftarrow i_{col}+size_v$\\
            $L_{i+1} \leftarrow L_{i+1}\cup \{ \varnothing,leaf(v)\}$
          }\Else{
            $\;\; I(i,i_{col}:i_{col}) \leftarrow \varnothing$\\
            $i_{col} \leftarrow i_{col}+1$\\
            $L_{i+1} \leftarrow L_{i+1}\cup \{ \varnothing\}$
          }
      }
    }
    \Return $I$
	\caption{{$f_2: \mathcal{T} \rightarrow \mathcal{I}$ Projection from Trees to Images}
	\label{Alg:f2}}
\end{algorithm}

\section{TreeRNN: a 2D RNN Network on Graph-tree-image}
The project function $f_2$ successfully transfers a graph from tree space to image space. In this section, we propose $f_3: \mathcal{I} \rightarrow \mathcal{Y}$ that extracts the topological features from the graph images and classifies their categories. 

There exist several approaches to $f_3$. We can choose widely used image classifiers such as ResNet \cite{he2016resnet} and GoogleNet \cite{szegedy2015googlenet} that are powerful, robust to extract features from images and estimate their categories. However, they are designed for images captured by cameras without considering the tree structure implied in the graph-tree-images. In ordered neurons \cite{shen2019ordered} an RNN named ON-LSTM is proposed to learn these context feature by processing the graph image column by column. With the help of its customized activation function, the recurrent unit resets its neurons before it starts the next segments. This RNN structure is specifically designed for tree images and achieved impressive results in the experiments. Unfortunately, the ON-LSTM is designed for a tree within 3 layers, which is difficult to work on deep trees.

\begin{figure*}[h]
\centering
\includegraphics[width=\textwidth]{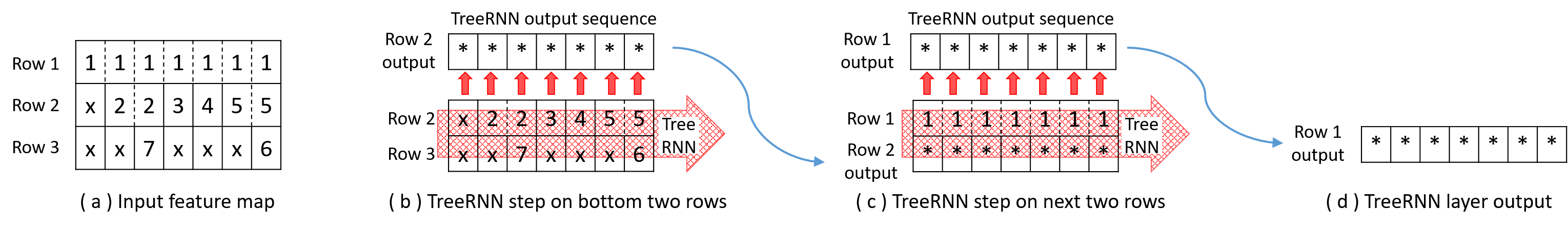}
 % \vspace{-1em}
\caption{{Illustration of TreeRNN operation in $f_3:$ network architecture. (a) a given image-like feature map from the previous network layer; (b) TreeRNN's first step on the feature map, which integrates pixels in the last two rows and return output sequence; (c) TreeRNN's next step on the feature map, which integrates pixels in the output of the last TreeRNN step and the next row of the feature map; (d) TreeRNN keeps integrating the output row from the last step and the input row from the feature map until it goes through all the rows in the input feature map, and eventually returns an output sequence. }\label{fig:f3_network}}
\end{figure*}
 % \vspace{-1em}

\textbf{TreeRNN.}
In our work, we propose TreeRNN, a 2D RNN architecture that is optimized for graph-tree-images. The TreeRNN takes existing RNN unit as kernels but works on rows and columns in turns on a 2D feature map. By working along with the rows, the RNN unit goes along with the tree layers to integrates the graph nodes within the same layer; by working along with the columns, it goes across the layers to integrate the nodes with their parents. Algorithm \ref{Alg:f3_RNN} presents the scheme of the proposed TreeRNN. Figure \ref{fig:f3_network} also illustrates how the RNN unit recurrently process the pixels by rows and columns.

 % \vspace{-0.25em}
\begin{algorithm}[h]\footnotesize
    \KwIn{Input tensor $ In(u,v,c) $ with size $ H \times W \times C $, RNN}
    \KwOut{Output tensor $ Out(u,v,c1) $ with size $ H \times W \times C1 $}
    \quad \\
    $ S \leftarrow RNN\left( In(1:2,\;:\;,\;:\;) \right) $\\
    \ForAll{$ u \in 3,4,\dots,H$}{
        $ \;\; S \leftarrow RNN\left( Concat(S,In(u,\;:\;,\;:\;)) \right) $
    }
    $ \;\; Out \leftarrow S $\\
    \Return $Out$
	\caption{{$f_3\_RNN:$ 2D RNN working scheme}
	\label{Alg:f3_RNN}}
\end{algorithm}
 % \vspace{-0.25em}

During processing, the RNN unit has to identify the "brother" nodes sharing the same parent node in the tree and separate them from with "cousin" nodes that share the same grand parent or great-grand parent node. In Ordered Neurons \cite{shen2019ordered} a master activation and a customized LSTM layer are proposed to force clear states after integrating a segment of "brother" nodes. In our work, we write the identification inside the graph-tree-images. As mentioned in Section \ref{subsec:f2}, in each row, we reserve a pixel before placing a segment of "brother" nodes, which separates this segment with others. Additionally, the gap between the segments exceeds one pixel if they are "distant relatives". Hence, the relationship between the graph nodes is clearly embedded in the graph images so that it can be easily learned by RNN.

\section{Experiments}
\subsection{Experimental Setup}
\textbf{Datasets.}
We evaluate our method \emph{i.e.} tree construction + image representation + network classifier, on three medium-size datasets for graph classification, namely MUTAG, PTC-MR and NCI1. Table \ref{tab:statistics} summarizes some statistics of each dataset.

\begin{table}[h]
	\caption{\footnotesize Statistics of benchmark datasets for graph classification.}
	\label{tab:statistics}
		\centering
		\begin{tabular}{c|cccc}
			\toprule
			Dataset & \begin{tabular}[c]{@{}c@{}}Num. of\\Graph\end{tabular} & \begin{tabular}[c]{@{}c@{}}Num. of\\Class\end{tabular} & \begin{tabular}[c]{@{}c@{}}Avg.\\Node\end{tabular} & \begin{tabular}[c]{@{}c@{}}Avg.\\Edge\end{tabular} \\
			\midrule
			MUTAG  & 188 & 2 & 17.93 & 19.79\\
			PTC-MR & 344 & 2 & 14.29 & 14.69 \\
			NCI1   & 4110 & 2 & 29.87 & 32.30 \\
			\bottomrule
		\end{tabular}
%	 % \vspace{-3mm}
\end{table}

\textbf{Implementation.}
By default, we design a simple network for $f_3$ to demonstrate the success of our graph-tree-image projection and TreeRNN. The network is "MLP $\rightarrow$ TreeRNN $\rightarrow$ MaxPool $\rightarrow$ FC", where MLP denotes a point-wised multi-layer perceptron and FC denotes a fully-connected layer. By default, we utilize a point-wised single-layer perceptron with 64 neurons and relu activation for the MLP block, a vanilla RNN unit with 64 neurons for the TreeRNN operator, and FC layer is set to have a softmax activation and an output size of $|\mathcal{Y}|$, the size of categories in the graph set ${\mathcal{G}}$.

We implement the scheme in a GPU machine with an Intel Core i5-6500 CPU and an NVidia GTX1060 GPU. The implement environment include the following key packages: Python 3.7, Networkx 2.4 \cite{hagberg2008Networkx}, and Tensorflow 2.2 \cite{tensorflow2015tensorflow}.

\begin{table*}[h]
\caption{Graph classification results (\%) in MUTAG, PTC-MR and NCI1. Numbers in red are the best in the column, and numbers in blue are the second best.}
\label{table:SOTA}
\centering
\begin{tabular}{| c | c | c | c | c |} 
 \toprule
 Category & Method & MUTAG & PTC-MR & NCI1 \\ 
 \hline
 \multirow{4}{5em}{Graph Convolution} 
 & GraphConv \cite{morris2019GraphConv} & 86.1 & - & 76.2 \\
 & GINConv \cite{xu2018GINConv} & \textcolor{red}{95.00 $\pm$ 4.61} & 72.94 $\pm$ 6.28 & 80.32 $\pm$ 1.73 \\
 & ECConv \cite{simonovsky2017ECConv} & 89.44 & - & 83.80 \\
 & DGCNN \cite{zhang2018DGCNN} & 85.83 $\pm$ 1.66  & 58.59 $\pm$ 2.47 & 74.44 $\pm$ 0.47  \\
 & GIC \cite{jiang2019gic} & 94.44 $\pm$ 4.30
 &  \textcolor{red}{77.64 $\pm$ 6.98} & 84.08 $\pm$1.07 \\
 \midrule
 \multirow{4}{5em}{Graph Embedding} 
 & PSCN \cite{niepert2016PSCN} & 88.95 $\pm$ 4.37 & 62.29 $\pm$ 5.68 & 76.34 $\pm$ 1.68 \\
 & DDGK \cite{al2019ddgk} & 91.58 $\pm$ 6.74 & 63.14 $\pm$ 6.57 & 68.10 $\pm$ 2.30 \\
 & WKPI \cite{zhao2019WKPI} & 85.8 $\pm$ 2.5 & 62.7 $\pm$ 2.7 & \textcolor{red}{87.5 $\pm$ 0.5} \\
 & \textbf{Ours} & \textcolor{blue}{94.74 $\pm$ 5.55} & \textcolor{blue}{74.69 $\pm$ 5.78} & \textcolor{blue}{84.96 $\pm$ 4.81} \\
 \bottomrule
\end{tabular}
\end{table*}

\textbf{Experimental Scheme.}
By default, during the experiment on each dataset, we separate the dataset into 10 folds, in which the samples within each category are evenly distributed. At each time we train within 9 folds and test within the last fold. During training, we use Adam \cite{duchi2011Adam} with learning rate $lr = 10^{-4}$ as the network optimizer. We then record the best accuracy of each fold, and calculate the mean accuracy as well as the standard deviation.

\subsection{Graph Classification}
To do a fair comparison for graph classification, we follow the standard routine, \emph{i.e.} 10-fold cross-validation with a random split.
In each fold, we have the same number of samples in each graph category. In this experiment, we apply an end-to-end training scheme that consistently inputs graph samples, generates graph-tree-images with data augmentation and feeds into network training steps. For each fold, we train the network 500 epochs and record the best testing result. In Table \ref{table:SOTA} we compare our method with several existing works on graph leaning.

The results show that our method achieves the second best accuracy in all three datasets. To emphasize
, our work results in the best accuracy on MUTAG and PTC-MR dataset among the graph embedding solutions. The small variances indicate the stability of our method. In summary, such results demonstrate the success of our method on graph classification.

\subsection{Ablation Study}
\textbf{Effects of Data Augmentation using Grid Layouts on Classification.}
In order to train the deep classifiers well, the amount of training data is crucial. In this paper we add two data augmentation: (1) in Algorithm \ref{Alg:f1} we randomly cut graph cycles during breadth first search (BFS) as illustrated in Figure \ref{fig:BFS}, and (2) in Algorithm \ref{Alg:f2} we shuffle the $leaf(v)$ so that the leaf nodes are projected to the image in a different order. In Table \ref{table:augmentation}, we demonstrate the test performance of network models trained 50 epochs on MUTAG dataset using 1$\times$, 6$\times$, and 11$\times$ data augmentation. In Table \ref{table:augmentation}, we observe that the data augmentation significantly increases the classification accuracy on MUTAG. Similar observations have been made for the other datasets.
\begin{table}[h!]
\caption{Effects of data augmentation on MUTAG accuracy}
 % \vspace{-0.5em}
\label{table:augmentation}
\centering
\begin{tabular}{| c | c | c | c |} 
 \toprule
 Augmentation (times) & $1\times$ & $6\times$ & $11\times$ \\ 
 \hline
 Accuracy (\%) & 84.62 & 89.42 & 92.54 \\
 \bottomrule
\end{tabular}
\end{table}

\begin{figure}[t]
	\begin{minipage}[b]{0.45\columnwidth}
		\begin{center}
			\centerline{\includegraphics[width=\columnwidth]{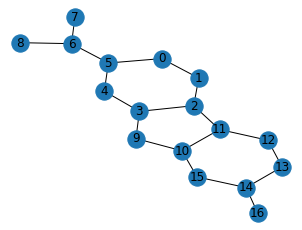}}
			\centerline{\footnotesize (a) Input Graph}
		\end{center}
	\end{minipage}
	\hfill
	\begin{minipage}[b]{0.45\columnwidth}
		\begin{center}
			\centerline{\includegraphics[width=\columnwidth]{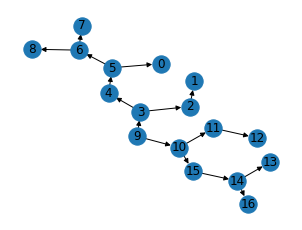}}
			\centerline{\footnotesize (b) BFS Tree 1}
		\end{center}
	\end{minipage}
	\hfill
	\begin{minipage}[b]{0.45\columnwidth}
		\begin{center}
			\centerline{\includegraphics[width=\columnwidth]{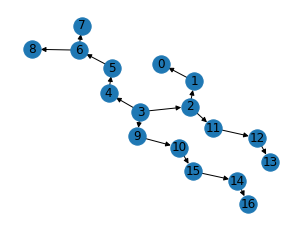}}
			\centerline{\footnotesize (b) BFS Tree 2}
		\end{center}
	\end{minipage}
	\hfill
	\begin{minipage}[b]{0.45\columnwidth}
		\begin{center}
			\centerline{\includegraphics[width=\columnwidth]{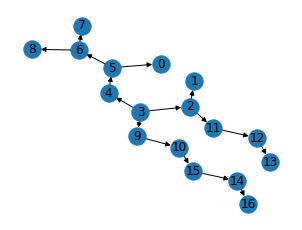}}
			\centerline{\footnotesize (b) BFS Tree 3}
		\end{center}
	\end{minipage}
	\caption{Illustration of several trees constructed using breath first search (BFS). (a) a graph sample from MUTAG dataset. (b,c,d) three BFS trees corresponding to the graph sample. \label{fig:BFS}}
\end{figure}

\textbf{Effects of Graph-Tree-Image Projection.} 
To understand the effectiveness of our proposed projection, we compare the classification results with the DeepWalk \cite{perozzi2014deepwalk} image projection in MUTAG dataset using the same network classifier. In this experiment we do not apply any data augmentation. The result shows that our projection gets 84.21\% $\pm$ 5.12\% in accuracy with our network classifier, while the DeepWalk projection with the same image size results in 78.32 \% $\pm$ 9.51 \% using the same classifier. The result indicates that our project better encodes the topological features in the graphs.
\begin{table}[h]
\caption{Comparison of MUTAG accuracy between TreeRNN and other neural network operators.}
 % \vspace{-0.5em}
\label{table:RNN}
\centering
\begin{tabular}{|c | c | c |} 
 \toprule
 Network Operator & Accuracy (\%) & Parameters\\ 
 \hline
 MLP & 77.81 & 6,024  \\ 
 2DConv & 81.55 & 38,792 \\
 RNN & 82.51 & 42,888 \\
 D-RNN & 80.91 & 10,120  \\
 \hline
 \textbf{TreeRNN} & \textbf{84.62} & 14,216 \\ 
 \bottomrule
\end{tabular}
\end{table}
 % \vspace{-1em}

\textbf{Effects of TreeRNN.}
To understand the effectiveness of our proposed TreeRNN as a feature extractor, we compare it with (1) a multi-layer perceptron (MLP), (2) a 2D convolution layer with $3 \times 3$ kernel (2DConv), (3) a conventional RNN layer (RNN), and (4) a distributed RNN layer (D-RNN) as described in \cite{lyu2019D_LSTM}. All feature extractors are implemented in the same classification network with the same feature size. In this experiment we do not apply any data augmentation. Table \ref{table:RNN} presents the comparison of various feature extractors in the MUTAG graph classification dataset. We observe that our TreeRNN achieves the best result in accuracy.

\section{Conclusion}
In this paper, we answer the question positively that graphs can be ordered to trees and further projected to image space so that order-sensitive network operators can benefit from its order and regularity. To this end, we propose a novel graph-tree-image projection, that projects a graph to image space while preserving its topological features between graph nodes. In addition, we propose TreeRNN, a 2D RNN scheme, that integrates the graph images simultaneously along with the tree layers and across the tree layers. In the experiment, we demonstrate that our work has comparable performance to the start-of-the-art in three datasets. As future work, we are interested in applying this method to real-world problems such as point cloud classification and segmentation.

{
\bibliographystyle{IEEEtran}
\bibliography{IEEEfull.bib}
}

\end{document}